\documentclass[letterpaper]{article}
\usepackage{ijcai15}

\usepackage{times}

\usepackage{graphicx} 

\usepackage{algorithm}
\usepackage[noend]{algpseudocode}

\usepackage{amsmath}
\usepackage{amsthm}
\usepackage{amsfonts}
\usepackage{amssymb}

\usepackage{breqn}
\usepackage{float}

\usepackage{url}

\usepackage[titletoc,title]{appendix}

\newcommand{\LB}[1]{\underline{#1}}
\newcommand{\UB}[1]{\overline{#1}}

\newcommand{\condOne}{(C1)}
\newcommand{\condTwo}{(C2)}
\newcommand{\RDIS}[1]{RDIS$_{#1}$}

\newtheorem{define}{Definition}
\newtheorem{prop}{Proposition}

\newcommand{\x}{\mathbf{x}}

\newcommand{\mb}[1]{\mathbf{#1}}

\newcommand{\idx}{\mathcal{I}}

\newcommand{\tuple}[1]{\ensuremath{\left \langle #1 \right \rangle }}
\newcommand{\tuplesm}[1]{\ensuremath{ \langle #1 \rangle }}

\newcommand{\citeAY}[1]{\citeauthor{#1}~[\citeyear{#1}]}
\newcommand{\citeNP}[1]{~\citeauthor{#1}~[\citeyear{#1}]}

\title{Recursive Decomposition for Nonconvex Optimization}
\author{Abram L. Friesen and Pedro Domingos \\
Department of Computer Science and Engineering \\
University of Washington\\
Seattle, WA 98195, USA \\
\normalsize{\texttt{\{afriesen,pedrod\}@cs.washington.edu}}}

\begin{document}

\maketitle 

%

\begin{abstract}
Continuous optimization is an important problem in many areas of AI,
including vision, robotics, probabilistic inference, and machine
learning. Unfortunately, most real-world optimization problems are
nonconvex, causing standard convex techniques to find only local
optima, even with extensions like random restarts and simulated
annealing. We observe that, in many cases, the local modes of the
objective function have combinatorial structure, and thus ideas from
combinatorial optimization can be brought to bear. Based on this, we
propose a problem-decomposition approach to nonconvex optimization.
Similarly to DPLL-style SAT solvers and recursive conditioning in
probabilistic inference, our algorithm, RDIS, recursively sets variables
so as to simplify and decompose the objective function into
approximately independent sub-functions, until the remaining functions
are simple enough to be optimized by standard techniques like gradient
descent. The variables to set are chosen by graph partitioning,
ensuring decomposition whenever possible. We show analytically that
RDIS can solve a broad class of nonconvex optimization problems
exponentially faster than gradient descent with random restarts. 
Experimentally, 
RDIS outperforms standard techniques
on problems like structure from motion and protein folding.
\end{abstract}

\section{Introduction}

AI systems that interact with the real world often have to solve continuous
optimization problems. For convex problems, which have no local optima,
many sophisticated algorithms exist. However, most continuous optimization
problems in AI and related fields are nonconvex, and often have an exponential
number of local optima. For these problems, the standard solution is to
apply convex optimizers with multi-start and other randomization
techniques~\cite{fabio1991}, but in problems with an exponential number of
optima these typically fail to find the global optimum in a
reasonable amount of time. 
Branch and bound methods can also be used, but
scale poorly due to the curse of dimensionality~\cite{neumaier2005}.

In this paper we propose that such problems can instead be approached using
problem decomposition techniques, which have a long and successful history
in AI for solving discrete problems
(e.g,~\cite{davis1962machine,darwiche01recursive,bayardo00,sang2004combining,sang2005performing,Bacchus2009}).
By repeatedly decomposing a problem into independently solvable
subproblems, these algorithms can often solve in polynomial time problems
that would otherwise take exponential time. The main difficulty in
nonconvex optimization is the combinatorial structure of the modes, which
convex optimization and randomization are ill-equipped to deal with, but
problem decomposition techniques are well suited to. We thus propose a
novel nonconvex optimization algorithm, which uses recursive decomposition
to handle the hard combinatorial core of the problem, leaving a set of
simpler subproblems that can be solved using standard continuous
optimizers.

The main challenges in applying problem decomposition to continuous
problems are extending them to handle continuous values and defining an
appropriate notion of local structure. We do the former by embedding
continuous optimizers within the problem decomposition search, in a manner
reminiscent of satisfiability modulo theory solvers~\cite{DeMoura2011}, but
for continuous optimization, not decision problems.
We do the latter by
observing that many continuous objective functions are approximately locally
decomposable, in the sense that setting a subset of the variables causes
the rest to break up into subsets that can be optimized nearly independently.
This is particularly true when the objective function is a sum of terms over
subsets of the variables, as is typically the case. A number of continuous
optimization techniques employ a static, global decomposition (e.g., block
coordinate descent~\cite{nocedal06} and partially separable
methods~\cite{griewank1981partiallysep}), but many problems only decompose
locally and dynamically, which our algorithm accomplishes.
%

For example, consider protein folding~\cite{anfinsen73,baker00}, the
process by which a protein, consisting of a chain of amino acids, assumes
its functional shape. The computational problem is to predict this final
conformation by minimizing a highly nonconvex energy function consisting
mainly of a sum of pairwise distance-based terms representing chemical
bonds, electrostatic forces, etc. Physically, in any conformation, an atom
can only be near a small number of other atoms and must be far from the
rest; thus, many terms are negligible in any specific conformation, but
each term is non-negligible in some conformation. This suggests that sections
of the protein could be optimized independently if the terms connecting
them were negligible but that, at a global level, this is never
true. However, if the positions of a few key atoms are set appropriately
then certain amino acids will never interact, making it possible to
decompose the problem into multiple independent subproblems and solve each
separately. A local recursive decomposition algorithm for continuous
problems can do exactly this.

We first define local structure and then present our algorithm, RDIS,
which (asymptotically) finds the global optimum of a nonconvex function
by (R)ecursively (D)ecomposing the function into locally (I)ndependent
(S)ubspaces. In our analysis, we show that RDIS achieves an exponential
speedup versus traditional techniques for nonconvex optimization such as
gradient descent with restarts and grid search (although the complexity
remains exponential, in general). This result is supported empirically, as
RDIS significantly outperforms standard nonconvex optimization algorithms
on three challenging domains: structure from motion, highly multimodal
test functions, and protein folding. \looseness=-1

\section{Recursive Decomposition for Continuous Optimization}

This section presents our nonconvex optimization algorithm, RDIS. We first present our notation
and then define local structure and a method for realizing it. We then describe RDIS and provide
pseudocode. 

In unconstrained optimization, the goal is to minimize an objective function $f(\x)$ over
the variables $\x \in \mathbb{R}^n$. We focus on functions 
${f : \mathbb{R}^n\rightarrow\mathbb{R}}$ that are continuously differentiable and have a 
nonempty optimal set $\x^*$ with optimal value $f^* = f(\x^*) > -\infty$. Let $\idx = \{ 1, \dots, n \}$ 
be the indices of $\x$, let $C \subseteq \idx$, let $\x_C \in \mathbb{R}^{|C|}$ be the 
restriction of $\x$ to the indices in $C$, and let $\rho_C \in \textup{domain}(\x_C)$ be
a partial assignment where only the variables corresponding to the
indices in $C$ are assigned values. We define $\x|_{\rho_C} \in \mathbb{R}^{n - |C|}$ 
to be the subspace where those variables with indices in $C$ are set to the values in $\rho_C$
(i.e., for some $\rho_C$ and for all $i \in C$ we have $\x|_{\rho_C,i} = \rho_{C,i}$).
Given $U = \idx \backslash C$ and partial
assignment $\rho_C$, then, with a slight abuse of notation, we define the restriction of the function 
to the domain $\x|_{\rho_C}$ as $f|_{\rho_C}(\x_U)$. In the following, we directly
partition 
$\x$ instead of discussing the partition 
of $\idx$ that induces it.

\subsection{Local structure}

%

A function is fully decomposable (separable) if it can be expressed as ${f(\x) = \sum_{i=1}^n{ g_i(x_i) } }$.
Such functions are easy to optimize, since they decompose with respect to minimization; i.e., 
${\min_{\x}{f(\x)} = \sum_{i=1}^n \min_{x_i}g_i(x_i) }$. Conversely, decomposable nonconvex functions that 
are optimized without first decomposing them require exponentially more exploration to find the
global optimum than the decomposed version. 
For example, let $M_f$ be the set of modes of $f$ and let $M_{i}$ be the modes of each $g_i$.
Knowing that $f$ is decomposable allows us to optimize each $g_i$ independently, giving
$|M_f| = \sum_{i=1}^n |M_{i}|$ modes to explore. However, if we instead
optimized $f$ directly, we would have to explore $\prod_{i=1}^n |M_{i}|$ modes,
which is exponential in $n$.
%
%
Unfortunately, fully decomposable functions like $f$ are rare,
as variables generally appear in multiple terms with many different variables
and thus the minimization 
does not trivially distribute. However, decomposition can still be achieved if the function
exhibits global or local structure, which we define here.
%
\begin{define}
~\\
\indent \emph{\textbf{(a)}} $f(\x)$ is \textbf{globally decomposable} if there exists a partition 
$\{\x_C, \x_{U_1}, \x_{U_2} \}$ of $\x$ such that,
for every partial assignment ${\rho_C}$, 
${f|_{\rho_C}(\x_{U_1}, \x_{U_2}) = f_1|_{\rho_C}(\x_{U_1}) + f_2|_{\rho_C}(\x_{U_2})}$.

\vspace{2.5pt}

\emph{\textbf{(b)}} $f(\x)$ is \textbf{locally decomposable} in the subspace  
$\x|_{\rho_C}$
if there exists a partition $\{\x_C, \x_{U_1}, \x_{U_2} \}$ of $\x$ 
and a partial assignment $\rho_C$ such that
${f|_{\rho_C}(\x_{U_1}, \x_{U_2}) = f_1|_{\rho_C}(\x_{U_1}) + f_2|_{\rho_C}(\x_{U_2})}$.

\vspace{2.5pt}

\emph{\textbf{(c)}} $f(\x)$ is \textbf{approximately locally decomposable} in a neighbourhood of
the subspace $\x|_{\rho_C}$
if there exists a partition $\{\x_C, \x_{U_1}, \x_{U_2} \}$ of $\x$, 
partial assignments $\rho_C, \sigma_C$,
and $\delta, \epsilon \geq 0$ such that 
if ${|| \sigma_C - \rho_C || \leq \delta}$ then
$\left| f|_{\sigma_C}(\x_{U_1},\x_{U_2}) - [f_1|_{\sigma_C}(\x_{U_1}) + f_2|_{\sigma_C}(\x_{U_2}) ]\right| \leq \epsilon$.
\label{def:structure}
\end{define}

Global structure (Definition~\ref{def:structure}a), while the easiest to exploit, is also the 
least prevalent. Local structure,
which may initially appear limited, subsumes global structure while also allowing
different decompositions throughout the space, making it strictly more general.
Similarly, approximate local structure
subsumes local structure.
In protein folding, for example, two amino acids may be pushed either close together or far 
apart for different configurations of other amino acids.
In the latter case, they can be optimized independently
because the terms connecting them are negligible. 
Thus, for different partial configurations of
the protein, different approximate decompositions are possible. 
The independent subspaces that result from local decomposition can themselves
exhibit local structure, allowing them to be decomposed in turn. If an algorithm exploits
local structure effectively, it never has to perform the full combinatorial optimization. 
Local structure does not need to exist everywhere in the space, just in the
regions being explored. For convenience, we only refer to local structure below,
unless the distinction between global or (approximate) local decomposition is relevant.

One method for achieving local decomposition 
is via (local) simplification.
We say that $f_i(\x_C, \x_U)$ is \emph{(approximately locally) simplifiable} in the subspace $\x|_{\rho_C}$
defined by partial assignment $\rho_C$ if, for a given $\epsilon \geq 0$, 
${ \UB{f_i}|_{\rho_C}(\x_U) - \LB{f_i}|_{\rho_C}(\x_U)  \leq 2\epsilon }$,
where $\UB{h}(\x)$ and $\LB{h}(\x)$ refer to the upper and lower bounds of $h(\x)$, respectively.
Similarly,
$f(\x)$ is \emph{(approximately locally) simplified} in the subspace $\x|_{\rho_C}$ 
defined by partial assignment $\rho_C$ 
if, for a given $\epsilon \geq 0$, all simplifiable terms $f_i|_{\rho_C}(\x_U)$ are replaced by the
constant ${k_i = \frac{1}{2}[\UB{f_i}|_{\rho_C}(\x_U) + \LB{f_i}|_{\rho_C}(\x_U)] }$.
%
%
%
%
For a function that is a sum of terms, local decomposition occurs when some of these terms simplify in
such a way that the minimization can distribute over independent groups of terms and variables (like 
component decomposition in Relsat~\cite{bayardo00} or in the protein folding example above). Given that
there are $m$ terms in the function, the maximum possible error in the simplified function versus the true
function is $m\cdot\epsilon$. However, this would require all terms to be simplified and their true values to be at
one of their bounds, which is extremely unlikely; rather, errors in different terms often cancel, and the 
simplified function tends to remain accurate. Note that $\epsilon$ induces a tradeoff between
acceptable error in the function evaluation and the computational cost of optimization, since a simplified function
has fewer terms and thus evaluating it and computing its gradient are both cheaper.
While the above definition is for sums of terms, 
the same mechanism applies to functions that are products
of (non-negative) factors, although error grows multiplicatively here.


\subsection{The RDIS Algorithm}

RDIS is an optimization method that explicitly finds and exploits 
local decomposition. Pseudocode is shown in Algorithm~\ref{alg:rdis}, with subroutines
explained in the text.
At each level of recursion, RDIS chooses a subset of the variables $\x_C \subseteq \x$
(inducing a partition $\{ \x_C, \x_U \}$ of $\x$)
and assigns them values $\rho_C$ such that
the simplified objective function $f|_{\rho_C}(\x_U)$ decomposes into multiple (approximately)
independent sub-functions 
$f_i|_{\rho_C}(\x_{U_i})$, where $\{ \x_{U_1}, \dots, \x_{U_k} \}$ is a partition of $\x_U$ and
$1 \leq k \leq n$. 
RDIS then recurses on each sub-function, globally optimizing it conditioned on
the assignment $\x_C = \rho_C$. When the recursion completes, 
RDIS uses the returned optimal values (conditioned on $\rho_C$) of $\x_U$ to choose new values for 
$\x_C$ and then simplifies, decomposes, and optimizes the function again. This repeats until 
a heuristic stopping criterion is satisfied.

\newcommand*\Let[2]{{#1} $\gets$ {#2}}
\algrenewcommand\algorithmicrequire{\textbf{Input:}}
\algrenewcommand\algorithmicensure{\textbf{Output:}}
\renewcommand{\algorithmiccomment}[1]{\bgroup\hfill\footnotesize //~\emph{#1}\egroup}

\algnewcommand{\IIf}[1]{\State\algorithmicif\ #1\ \algorithmicthen}
\algnewcommand{\EndIIf}{\unskip\ }

\begin{algorithm}[h]
   \caption{Recursive Decomposition into locally Independent Subspaces (RDIS). 
   }
   \label{alg:rdis}
\begin{algorithmic}[1]
 \Require{Function $f$, variables $\x$, initial state $\x^0$,
        subspace optimizer $S$, and approximation error $\epsilon$.}
 \Ensure (Approximate) global minimum $f^*$ at state $\x^*.$
 \Function{RDIS}{$f, \x, \x^0, S, \epsilon$}
   \State \Let{$\x_C$}{\Call{chooseVars}{$\x$}} \label{line:choosevars} \Comment{variable selection} 
   \State $\x_U \gets \x \backslash \x_C, ~~f^* \gets \infty, ~~\x^* \gets \x^0$
   \Repeat \label{line:loop}
     \State{partition $\x^*$ into $\{ \sigma^*_C, \sigma^*_U \}$}
     \State{$\rho_C \gets {S}(~f|_{\sigma^*_U}(\x_C ), \sigma^*_C~)$}  \label{line:ssopt}  \Comment{value selection}
     \State $\hat{f}|_{\rho_C}(\x_U) \gets$ \Call{simplify}{$f|_{\rho_C}(\x_U), \epsilon$} \label{line:simplify} 
     \State $\{\hat{f}_{i}(\x_{U_i})\} \gets$ \Call{decompose}{$\hat{f}|_{\rho_C}(\x_U)$} \label{line:decompose} 
     \For{i = 1, \dots, k} \Comment{recurse on the components}
        \State $\tuplesm{f_i^*,\rho_{U_i}} \gets$ \Call{RDIS}{$\hat{f}_{i},\x_{U_i}, \sigma_{U_i}^*, S, \epsilon$} 
            \label{line:recurse} 
     \EndFor
     \State $f^*_\rho \gets \sum_{i=1}^k f^*_i, ~~\rho_U \gets \cup_{i=1}^k{ \rho_{U_i} }$
     \vspace{1.2pt}
     \If{$f^*_\rho < f^*$} \Comment{record new minimum}
        \State $f^* \gets f^*_\rho, ~~\x^* \gets \rho_C \cup \rho_U$
     \EndIf
   \Until{ stopping criterion is satisfied }
   \State \Return $\tuple{f^*, \x^*}$
 \EndFunction
\end{algorithmic}
\end{algorithm}

RDIS selects variables (line~\ref{line:choosevars}) heuristically, with the 
goal of choosing a set of variables that 
enables the largest amount of decomposition, as this provides the largest computational gains.
Specifically, RDIS uses a hypergraph partitioning algorithm
to determine a small cutset that will decompose the graph; 
this cutset becomes the selected variables, $\x_C$.
Values for $\x_C$ are determined (line~\ref{line:ssopt}) by calling a nonconvex subspace 
optimizer with the remaining variables ($\x_U$) fixed to their current values.
The subspace optimizer $S$ is specified by the user and is customizable to the problem being solved.
In our experiments we used multi-start versions of conjugate gradient descent and 
Levenberg-Marquardt~\cite{nocedal06}. 
Restarts occur within line~\ref{line:ssopt}: 
if $S$ converges without making progress then 
it restarts to a new point in $\x_C$ and runs until it reaches a local minimum.

\begin{figure}[t]
\centering
\centerline{\includegraphics[width=0.8\columnwidth]{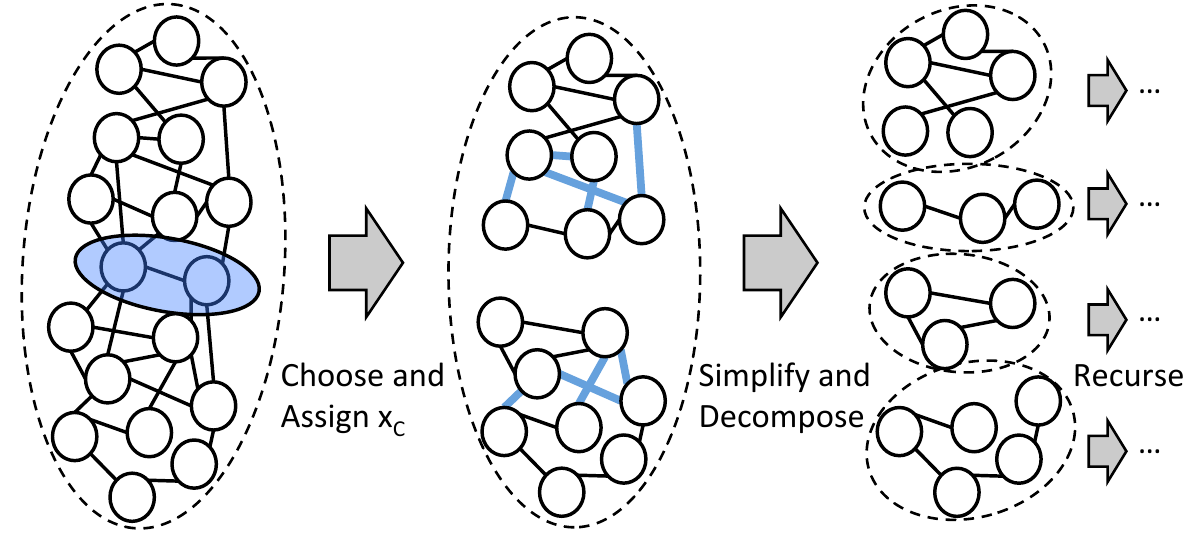}}
\caption{\small{Visualization of RDIS decomposing the objective function. 
Vertices (circles) represent variables and edges connect each pair of variables in a 
term.
Left: RDIS selects $\x_C$ (shaded oval). 
Middle: The function during simplification. 
Thick edges indicate simplifiable terms. 
Assigned variables are constant and have been removed. 
Right: The function after 
decomposition. 
}}
\label{fig:decomp}
\end{figure} 

To simplify the objective function (line~\ref{line:simplify}), RDIS determines
which terms are simplifiable (i.e., have sufficiently small bounds) and then simplifies (approximates)
these by replacing them with a constant. 
These terms
are not passed to the recursive calls.
After variables have been assigned and the function simplified, RDIS
locally decomposes (line~\ref{line:decompose}) the simplified function 
into independent sub-functions (components)
that have no overlapping terms or variables and thus can be optimized independently, 
which is done by recursively calling RDIS on each. 
See Figure~\ref{fig:decomp} for a visualization of this process. 
The recursion halts when 
\Call{ChooseVars}{} selects all of $\x$ (i.e., $\x_C = \x$ and $\x_U = \varnothing$), 
which occurs when
$\x$ is small enough that the subspace optimizer can optimize $f(\x)$ directly.
At this point, 
RDIS repeatedly calls the subspace optimizer until the
stopping criterion is met,
which (ideally) finds the global optimum of $f|_{\sigma^*_U}(\x_C) = f(\x)$
since $S$ is a nonconvex optimizer.
The stopping criterion is user-specified, and depends on the subspace optimizer.
If a multi-start descent method is used, termination
occurs after a specified number of restarts, corresponding to a certain probability that the
global optimum has been found. If the subspace optimizer is grid search, then the loop 
terminates after all values of $\x_C$ have been assigned. 
More subroutine details 
are provided in Section~\ref{sec:details}. 

\section{Analysis}
\label{sec:analysis}

We now present analytical results demonstrating the benefits of RDIS versus standard
algorithms for nonconvex optimization. Formally, we show that RDIS 
explores the state space
in exponentially less time than the same subspace optimizer for a class of functions that 
are locally decomposable, and that it will (asymptotically) converge to the global optimum.
Due to space limitations, proofs are 
presented in Appendix~\ref{app:analysis}. 
Let the number of variables be $n$ and the number of values 
assigned by RDIS to $\x_C$ be $\xi(d)$, where $d$ is the size of $\x_C$. The form of
$\xi(d)$ depends on the subspace optimizer, but can be roughly interpreted as 
the number of modes of the sub-function $f|_{\sigma^*_U}(\x_C)$ to explore
times a constant factor.

\begin{prop}
If, at each level, RDIS chooses $\x_C \subseteq \x$ of size 
$| \x_C | = d$ such that, for each 
selected value $\rho_C$, the simplified function $\hat{f}|_{\rho_C}(\x_U)$
locally decomposes into $k > 1$ independent sub-functions $\{ \hat{f}_i(\x_{U_i}) \}$
with equal-sized domains 
$\x_{U_i}$,
then 
the time complexity of RDIS is $O(\frac{n}{d} \xi(d)^{\log_k{(n/d)}})$.
\label{thm:complexity}
\end{prop}
\noindent Note that since RDIS uses hypergraph partitioning to choose variables, it will
always decompose the remaining variables $\x_U$. This is also supported by our experimental results; if
there were no decomposition, RDIS would not perform any better than the baselines. 


From Proposition~\ref{thm:complexity}, we can compute the time complexity of RDIS for
different subspace optimizers. Let the subspace optimizer be grid search (GS) over a bounded
domain of width $w$ with spacing $\delta$ in each dimension. Then the complexity of
grid search is simply $O((w/\delta)^n) = O(s^n) $.
\begin{prop}
If the subspace optimizer is grid search, then $\xi(d) = (w/\delta)^d=s^d$, 
and the complexity of \RDIS{GS} is $O(\frac{n}{d} s^{d\log_k{(n/d)}})$.
\end{prop}
\noindent 
Rewriting the complexity of grid search as $O(s^n) = O(s^{d(n/d)})$, we see that it is exponentially
worse than the complexity of \RDIS{GS} when decomposition occurs.

Now consider a descent method with random restarts (DR) as the subspace optimizer.
Let the volume of the basin of attraction of the global minimum (the global basin) be $l^n$ and
the volume of the space be $L^n$. 
Then the probability of randomly restarting in the global basin
is $(l/L)^n = p^n$. Since the restart behavior of DR is a Bernoulli process, the expected
number of restarts to reach the global basin is $r = p^{-n}$, from the shifted geometric distribution.
%
If the number of iterations needed to reach the stationary point of 
the current basin is $\tau$ 
then the expected complexity of DR is 
$O(\tau p^{-n})$. 
If DR is used within RDIS, then
we obtain the following result.
\begin{prop}
If the subspace optimizer is DR, then the expected value of $\xi(d)$ is 
$\tau p^{-d}$, and the expected complexity of \RDIS{DR} is 
$O(\frac{n}{d} (\tau p^{-d})^{\log_k{(n/d)}})$.
\end{prop}
\noindent
Rewriting the expected complexity of DR as $O(\tau {(p^{-d}})^{n/d})$ shows that \RDIS{DR} is 
exponentially more efficient than DR.

Regarding convergence, RDIS with $\epsilon = 0$ converges to the global minimum 
given certain conditions on the subspace optimizer. 
For grid search, \RDIS{GS} returns the global minimum if
the grid is finite and has sufficiently fine spacing.
For gradient descent with restarts, \RDIS{DR} will converge to stationary points of $f(\x)$
as long as steps by the subspace optimizer satisfy two technical conditions.
The first is an Armijo rule guaranteeing sufficient decrease in $f$ and the
second guarantees a sufficient decrease in the norm of the gradient 
(see \condOne~and \condTwo~in
Appendix~\ref{app:convergence}). 
These conditions are necessary to show that \RDIS{DR} behaves like an inexact Gauss-Seidel 
method~\cite{bonettini2011inexact}, and thus
each limit point of the generated sequence is a stationary point of $f(\x)$.
Given this, we can state the probability with which \RDIS{DR} will converge to the global minimum. 

\begin{prop}
If the non-restart steps of RDIS satisfy 
\condOne~and \condTwo, $\epsilon = 0$, the number of variables is $n$,
the volume of the global basin is $v = l^n$, and the volume of the entire space is $V = L^n$,
then \RDIS{DR} returns the global minimum after $t$ restarts, 
with probability $1 - (1-(v/V))^t$.
\end{prop}

For $\epsilon > 0$, we do not yet have a proof of convergence, even in 
the convex case, since preliminary analysis indicates that there are rare corner cases in which
the alternating aspect of RDIS, combined with the simplification error, can potentially
result in a non-converging sequence of values; however, we have not experienced this
in practice. Furthermore, our experiments clearly show $\epsilon > 0$ to be extremely
beneficial, especially for large, highly-connected problems.

Beyond its discrete counterparts, RDIS is related to many well-known continuous optimization algorithms. 
If all variables are chosen at the top level of recursion, then RDIS simply reduces to executing the 
subspace optimizer. If one level of recursion occurs, then RDIS behaves 
similarly to alternating minimization algorithms 
(which also have global convergence results~\cite{grippo1999globally}). 
For multiple levels of recursion, RDIS has 
similarities to block coordinate (gradient) descent algorithms 
(see~\citeAY{tseng2009coordinate}
and references therein). However, what sets RDIS apart is
that decomposition in RDIS is determined locally, dynamically, and recursively.
Our analysis and experiments show that exploiting this 
can lead to substantial performance improvements.

\section{RDIS Subroutines}
\label{sec:details}

In this section, we present the specific choices we've made for the subroutines in 
RDIS, but note that others are possible and we intend
to investigate them in future work.

\noindent
\textbf{Variable Selection.} 
Many possible methods exist for choosing variables. 
For example, heuristics from satisfiability may be applicable 
(e.g., 
VSIDS~\cite{moskewicz2001chaff}).
However, RDIS uses hypergraph partitioning in order to ensure decomposition whenever possible.
Hypergraph partitioning splits a graph into $k$ components of approximately equal size while 
minimizing the 
number of hyperedges cut. To maximize decomposition, RDIS
should choose the smallest block of variables that, when assigned, decomposes the remaining 
variables. 
This corresponds exactly to the set of edges cut by hypergraph partitioning
on a hypergraph that has 
a vertex for each term
and a hyperedge for each variable that connects the terms that variable is in
(note that this is the inverse of Figure~\ref{fig:decomp}).
%
%
RDIS maintains such a hypergraph and uses 
the PaToH hypergraph partitioning library~\cite{patoh2011} to quickly find 
good, approximate partitions. A similar idea was used in 
\citeauthor{darwiche2001hypergraph}~\citeyear{darwiche2001hypergraph}
to construct d-trees for recursive conditioning; however, they only apply hypergraph partitioning 
once at the beginning, whereas RDIS performs it at each level of the recursion. 

While variable selection
could be placed inside the loop, it would repeatedly choose the same variables
because hypergraph partitioning is based on the graph structure.
However, RDIS still exploits local decomposition because the variables and terms 
at each level of recursion vary based on local structure. In addition, 
edge and vertex weights
could be set 
based on current bounds or other local information.

\noindent
\textbf{Value Selection.}
RDIS can use any nonconvex optimization subroutine to choose values, allowing the user to 
pick an optimizer appropriate to their domain.
In our experiments, we use multi-start versions of both conjugate gradient descent and 
Levenberg-Marquardt, but other possibilities include Monte Carlo search,
quasi-Newton methods, and simulated annealing. 
We experimented with both grid search and branch and bound, 
but found them practical only for easy problems.
%
In our experiments, we have found it helpful to stop
the subspace optimizer early, because values are likely to change again in the next iteration, making
quick, approximate improvement more effective than slow, exact improvement.

\noindent
\textbf{Simplification and Decomposition.} 
Simplification is performed by checking whether each term (or factor) is simplifiable
and, if it is, setting it to a constant and removing it from the function.
RDIS knows the analytical form of the function and uses 
interval arithmetic~\cite{hansen2003global} as a general
method for computing and maintaining bounds on terms to determine simplifiability.
RDIS maintains the connected
components of a dynamic 
graph~\cite{holm2001poly} over the variables 
and terms (equivalent in structure to a factor or co-occurrence graph). 
Components in RDIS correspond exactly to the 
connected components in this graph. 
Assigned variables and simplified terms are removed from this graph,
potentially inducing local decomposition.


\noindent
\textbf{Caching and Branch \& Bound.}
RDIS' similarity to model counting algorithms suggests the use of
component caching and branch and bound (BnB). 
We experimented with these and found them effective when used with grid 
search; however, they were not
beneficial when used with descent-based subspace 
optimizers, which dominate grid-search-based RDIS on non-trivial problems.
For caching, this is because components are almost never seen again, due to not 
re-encountering variable values, even approximately. For BnB, 
interval arithmetic bounds tended to be overly loose and no bounding occurred.
Our experience suggests that this is because the descent-based optimizer effectively 
focuses exploration on the minima of the space, which are typically close in value
to the current optimum.
However, we believe that future work on caching and better bounds
would be beneficial.

\section{Experimental Results}
\label{sec:expres}

We evaluated RDIS on three difficult nonconvex optimization problems with hundreds
to thousands of variables: structure from motion,
a high-dimensional sinusoid, and protein folding. Structure from motion is an important problem
in vision, while protein folding 
is a core problem in computational biology.
We ran RDIS with a fixed number of restarts at each level, thus not guaranteeing that
we found the global minimum.
For structure from motion, we compared RDIS to that domain's standard technique of 
Levenberg-Marquardt (LM)~\cite{nocedal06} using the levmar 
library~\cite{lourakis04LM}, as well as to a block-coordinate descent version (BCD-LM).
In protein folding, gradient-based methods 
are commonly used to determine the lowest energy configuration of a protein,
so we compared RDIS to conjugate gradient descent (CGD) and 
a block-coordinate descent version (BCD-CGD).
CGD and BCD-CGD were also used for the high-dimensional sinusoid.
Blocks were formed
by grouping contextually-relevant variables together (e.g., in protein folding, 
we never split up an amino acid).
We also compared to ablated versions of RDIS.
RDIS-RND uses a random variable selection heuristic 
and RDIS-NRR does not use any internal random restarts (i.e., it functions as
a convex optimizer) but does have top-level restarts. In each domain, the optimizer we
compare to was also used as the subspace optimizer in RDIS.
All experiments were run on the same cluster. Each computer in the cluster 
was identical, with two 2.33GHz quad core Intel Xeon E5345 processors and 16GB of RAM.
Each algorithm was limited to a single thread. 
Further details can be found in Appendix~\ref{app:experiments}. 

\noindent
\textbf{Structure from Motion.}
Structure from motion is the problem of reconstructing the geometry of a 3-D scene
from a set of 2-D images of that scene. It consists of first determining an initial estimate
of the parameters and then performing non-linear optimization to
minimize the squared error between a set of 2-D image points and a projection of the 3-D points
onto camera models~\cite{triggs00bundle}. The latter, known as bundle adjustment,
is the task we focus on here.
Global structure exists, since cameras interact 
explicitly with points, creating a bipartite graph structure that RDIS can decompose, but (nontrivial) local 
structure does not exist because the bounds on each term are too wide and tend to 
include $\infty$. 
The dataset 
used is the 49-camera, 7776-point data file from the Ladybug 
dataset~\cite{agarwal2010bundle}

\begin{figure}[tb]
\includegraphics[width=\columnwidth]{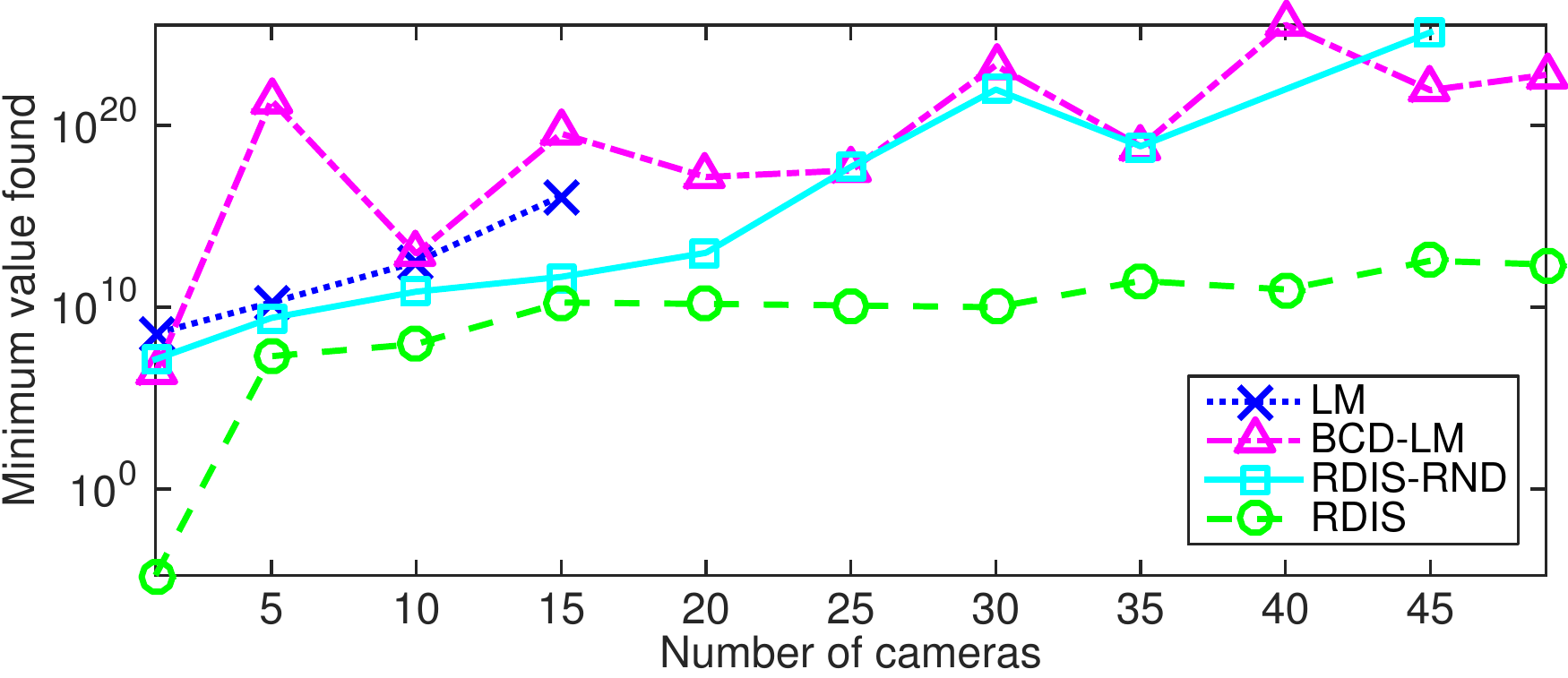}
\caption{\small{Minimum value found in one hour for increasing sizes of bundle adjustment 
problem (y-axis is log scale).}
}
\label{fig:bundle}
\end{figure} 

Figure~\ref{fig:bundle} shows performance on bundle adjustment
as a function of the size of the problem, with a log scale y-axis. Each point is the 
minimum error found after running each algorithm
for $5$ hours. Each algorithm is given the same set of restart states, but algorithms that converge faster 
may use more of these. 
Since no local structure is exploited, Figure~\ref{fig:bundle} effectively demonstrates the benefits of using
recursive decomposition with intelligent variable selection for nonconvex optimization. Decomposing
the optimization across independent subspaces allows the subspace optimizer to move faster, 
further, and more consistently, allowing
RDIS to dominate the other algorithms. 
Missing points are due to algorithms not returning any results in the allotted time.


\noindent
\textbf{High-dimensional Sinusoid.} 
The second domain is
a highly-multimodal test function defined as a multidimensional sinusoid placed in the basin of a
quadratic, with a small slope to make the global minimum unique. 
The arity of this function 
(i.e., the number of variables contained in each term) is controlled parametrically. 
Functions with larger arities contain more terms and dependencies, and thus 
are more challenging.
A small amount of local structure exists in this problem.


In Figure~\ref{fig:testfunc}, we show the current best value found versus time. 
Each datapoint is from a single run of an algorithm using the same set of top-level restarts,
although, again, algorithms that converge faster use more of these. RDIS outperforms all
other algorithms, including RDIS-NRR. This is due to the nested restart behavior afforded by
recursive decomposition, which allows RDIS to effectively explore each subspace and escape local minima.
The poor initial performance of RDIS for arities $8$ and $12$ is due to it being trapped
in a local minimum for an early variable assignment while performing optimizations lower in the recursion. 
However, once the low-level recursions finish it
escapes and finds the best minimum without ever performing a top level restart
(Figure~\ref{fig:trajectories} in Appendix~\ref{app:testfunc} contains the full trajectories).


%

\noindent
\textbf{Protein Folding.}
The final domain is sidechain placement for protein folding 
with continuous angles between atoms. Amino acids are composed
of a backbone segment and a sidechain. Sidechain placement requires setting the sidechain
angles with the backbone atoms fixed. It is equivalent to finding the MAP
assignment of a continuous pairwise Markov random field 
(cf.,~\citeauthor{yanover2006linear}~[\citeyear{yanover2006linear}]).
Significant local structure is present in this domain.
Test proteins were selected from the Protein Data Bank~\cite{Berman2000}
with sequence length 300-600 such that the sequences of any two
did not overlap by more than 30\%.


%

\begin{figure}[tb]
%
\includegraphics[width=\columnwidth]{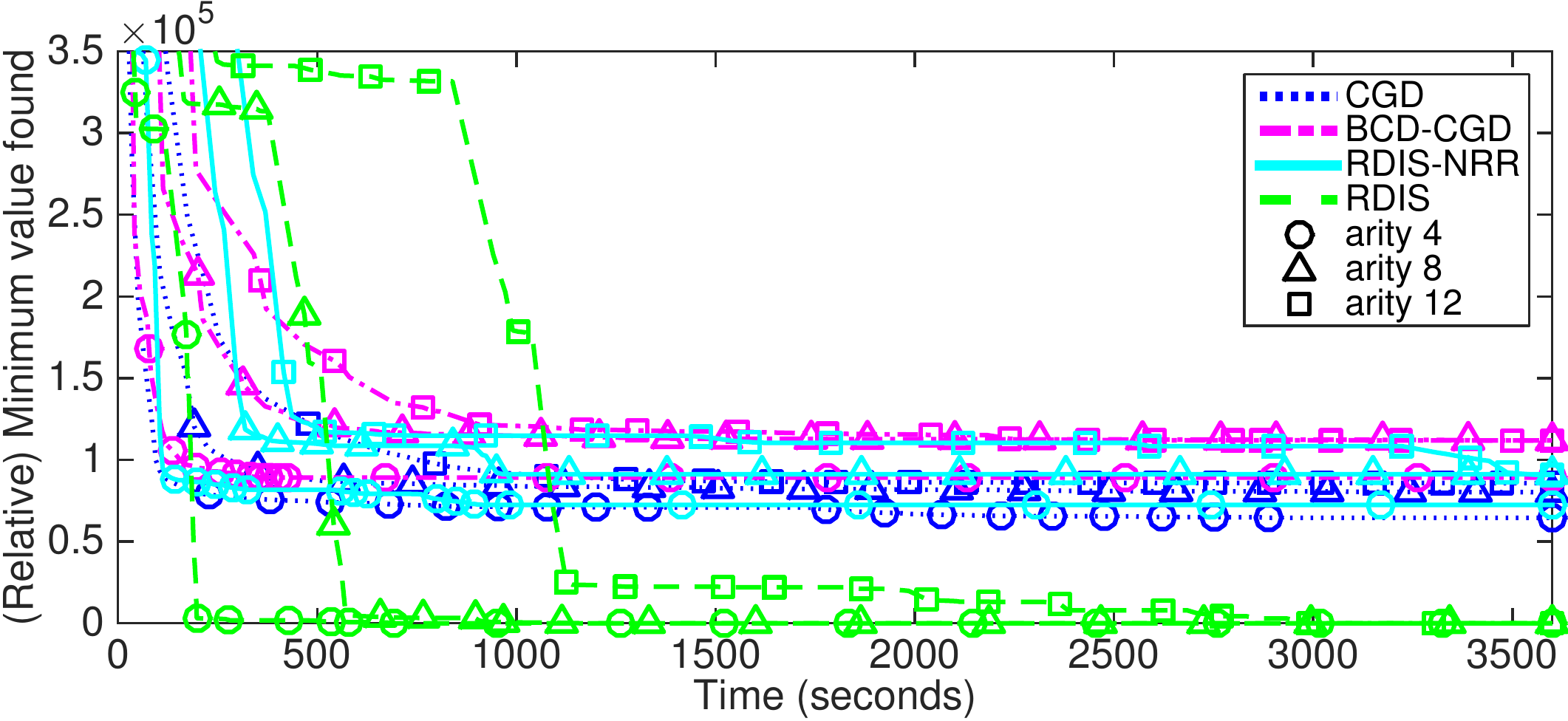}
\caption{\small{
A comparison of the best minima found as a function of time for three different 
arities of the high-dimensional sinusoid.
}}
\label{fig:testfunc}
\end{figure}


\begin{figure}[tb]
\includegraphics[width=\columnwidth]{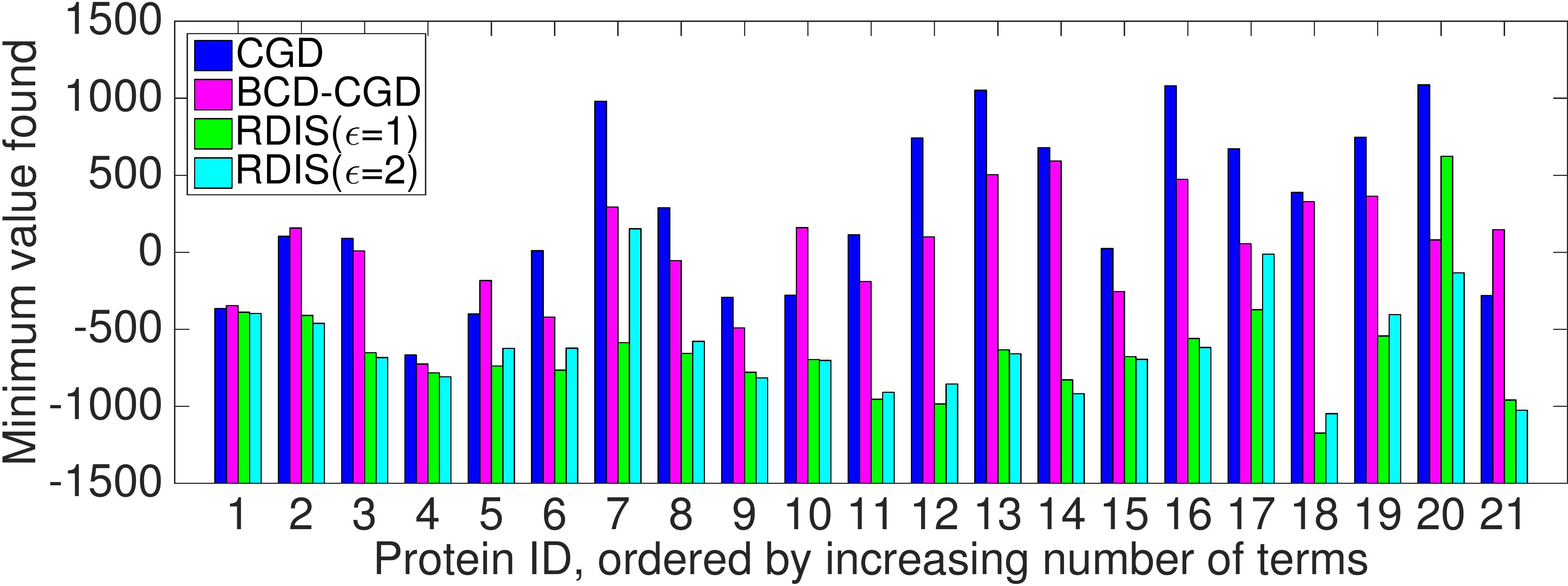}
\caption{\small{
Minimum value (energy) found on $21$ different proteins.
}}
\label{fig:proteins}
%
\end{figure} 

Figure~\ref{fig:proteins} shows the results of combining all aspects of RDIS, 
including recursive decomposition, 
intelligent variable selection, internal restarts, and local structure on a difficult problem with significant local 
structure. Each algorithm is run for $48$ hours on each of $21$ proteins of varying sizes. RDIS is run
with both $\epsilon = 1.0$ and $\epsilon = 2.0$ and both results are shown on the figure.
RDIS outperforms CGD and BCD-CGD on all proteins, often by a very large amount. 

Figure~\ref{fig:epsproteins} demonstrates the effect of $\epsilon$ on RDIS for protein folding. It
shows the performance of RDIS-NRR as a function of 
$\epsilon$, where performance is measured both by minimum 
energy found and time taken. 
RDIS-NRR is used in order to remove the randomness associated with the internal
restarts of RDIS, resulting in a more accurate comparison across multiple runs.
Each point on the 
energy curve is the minimum energy found over the same $20$ restarts. 
Each point on the time
curve is the total time taken for all $20$ restarts. As $\epsilon$ increases, 
time decreases because
more local structure is being exploited. 
In addition, minimum energy actually
decreases initially.  
We attribute this to the smoothing caused
by increased simplification, allowing RDIS to avoid minor local minima in the objective function.

\begin{figure}[t]
\includegraphics[width=\columnwidth]{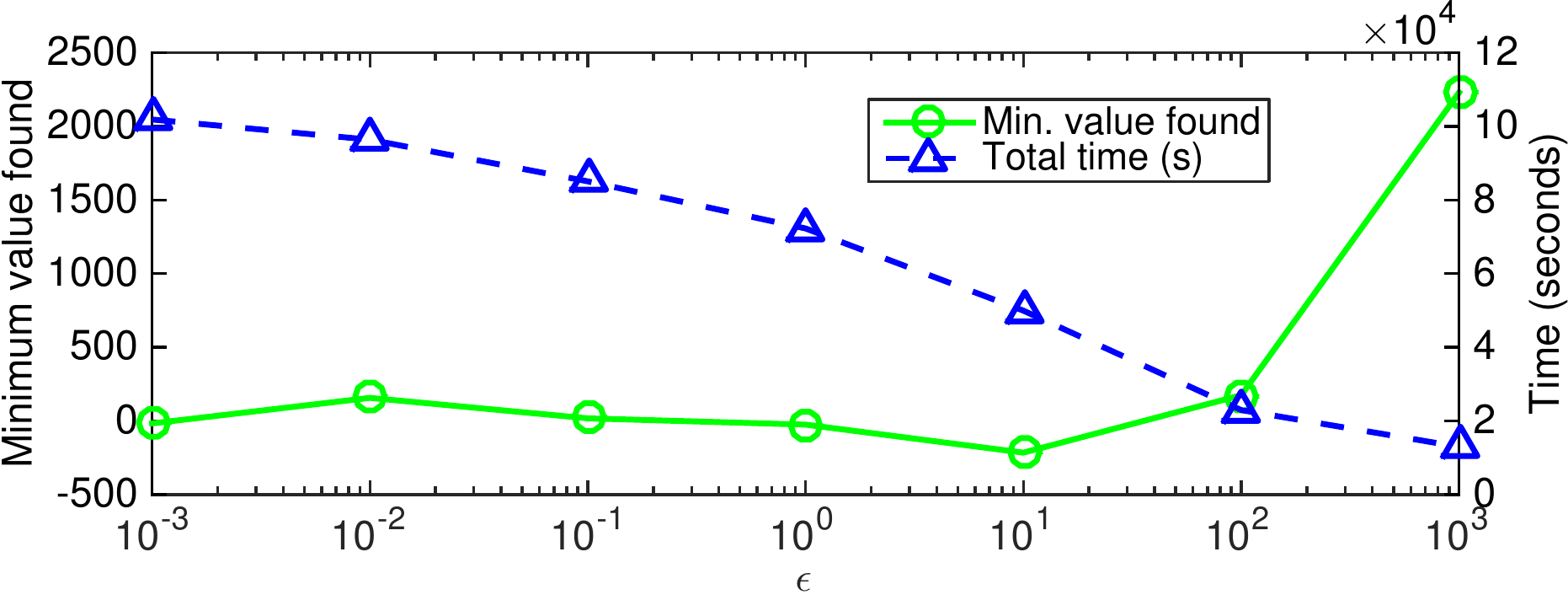}
\caption{\small{RDIS-NRR's minimum energy found and total time taken versus $\epsilon$, on 
protein 3EEQ (ID 9). The x-axis is log scale. 
}}
\label{fig:epsproteins}
\end{figure}

\section{Conclusion}

This paper proposed a new approach to solving hard nonconvex optimization problems based on 
recursive decomposition. 
RDIS decomposes the function into approximately locally
independent sub-functions and then optimizes these separately by recursing on them.
This results in an exponential reduction in the time required to find the global optimum.
In our experiments, we show that problem decomposition
enables RDIS to systematically outperform comparable methods.

Directions for future research include applying RDIS to a wide variety of nonconvex optimization problems,
further analyzing its theoretical properties, developing new variable and value selection methods,
extending RDIS to handle hard constraints, incorporating discrete variables, and
using similar ideas for high-dimensional integration.

\section*{Acknowledgments}

This research was partly funded by ARO grant W911NF-08-1-0242,
ONR grants N00014-13-1-0720 and N00014-12-1-0312, and AFRL contract
FA8750-13-2-0019. The views and conclusions contained in this
document are those of the authors and should not be interpreted as
necessarily representing the official policies, either expressed or
implied, of ARO, ONR, AFRL, or the United States Government.


%
%
%
%
%

\bibliographystyle{named}
\bibliography{ncopt}

\newpage
~
\newpage
\begin{appendices}

\renewcommand\thesection{\Alph{section}}


\section{Analysis Details}
\label{app:analysis}

\subsection{Complexity}

RDIS begins by choosing a block of variables, $\x_C$. Assuming that this choice is made 
heuristically using the PaToH library for hypergraph partitioning, which is a multi-level technique, 
then the complexity of choosing variables is linear 
(see \citeauthor{trifunovic2006parallel}~\citeyear{trifunovic2006parallel}, p.~81). 
Within the loop, RDIS chooses values for $\x_C$, 
simplifies and decomposes the function, and finally recurses. Let the complexity of choosing values 
using the subspace optimizer be $g(d)$, where $|\x_C| = d$, and let one call to
the subspace optimizer be cheap relative to $n$ (e.g., computing the gradient of $f$ with respect to $\x_C$
or taking a step on a grid).
Simplification requires iterating through the set of terms and computing bounds, so is linear in the
number of terms, $m$. The connected components are maintained by a dynamic graph 
algorithm~\cite{holm2001poly} which has an amortized complexity of~$O(\log^2(|V|))$ per operation, 
where $|V|$ is the number of vertices in the graph. Finally, let the number of iterations of 
the loop be a function of the dimension, $\xi(d)$, since more dimensions generally require
more restarts.
%
\begin{prop}
If, at each level, RDIS chooses $\x_C \subseteq \x$ of size 
$| \x_C | = d$ such that, for each 
selected value $\rho_C$, the simplified function $\hat{f}|_{\rho_C}(\x_U)$
locally decomposes into $k > 1$ independent sub-functions $\{ \hat{f}_i(\x_{U_i}) \}$
with equal-sized domains 
$\x_{U_i}$,
then 
the time complexity of RDIS is $O(\frac{n}{d} \xi(d)^{\log_k{(n/d)}})$.
\label{prop:complexity2}
\end{prop}
\begin{proof}
Assuming that $m$ is of the same 
order as $n$, the recurrence relation for RDIS is 
$T(n) = O(n) + \xi(d)\left[g(d) + O(m) + O(n) + O(d\log^2(n)) + k~T\left(\frac{n-d}{k}\right)\right]$, 
which can be simplified to $T(n) = \xi(d)\left[k~T\left(\frac{n}{k}\right) + O(n)\right] + O(n)$.
Noting that
the recursion halts at $T(d)$, the solution to the above recurrence relation is then
${T(n) = c_1~\left(k~\xi(d)\right)^{log_k{(n/d)}} + c_2~n \sum_{r=0}^{log_k{(n/d)}-1}{\xi(d)^r}}$.
which is $O\left((k\,\xi(d))^{\log_k{(n/d)}}\right) = O\left(\frac{n}{d} \xi(d)^{\log_k{(n/d)}}\right)$.
\end{proof}


\subsection{Convergence}
\label{app:convergence}

In the following, we refer to the basin of attraction of the global minimum as the global basin. Formally, we define
a basin of attraction as follows.
\begin{define}
The basin of attraction of a stationary point $\mb{c}$ is the set of points $B \subseteq \mathbb{R}^n$ for which the 
sequence generated by DR, initialized at $\x^0 \in B$, converges to $\mb{c}$.
\end{define}
Intuitively, at each level of recursion, RDIS with $\epsilon = 0$ partitions $\x$ into $\left\{\x_C,\x_U\right\}$, sets 
values using the subspace optimizer $\rho_C$ for $\x_C$, globally optimizes $f|_{\rho_C}(\x_U)$ by 
recursively calling RDIS, and repeats.
When the non-restart steps of the subspace optimizer satisfy
two practical conditions (below)
of sufficient decrease in (1) the objective function (a standard Armijo condition) and (2) the gradient 
norm over two successive partial updates,
(i.e., conditions (3.1) and (3.3) of~\citeNP{bonettini2011inexact}), 
then this process is equivalent to the $2$-block inexact Gauss-Seidel method (2B-IGS) described
in~\citeNP{bonettini2011inexact}
(c.f.,~\citeNP{grippo1999globally}; 
~\citeNP{cassioli2013convergence}), 
and each limit point of the sequence generated by RDIS is a stationary 
point of $f(\x)$, of which the global minimum is one, and reachable through restarts. 


Formally, let superscript $r$ indicate the recursion level, with $0 \leq r \leq d$, with $r = 0$ the top, and
recall that $\x_U^{(r)} = \left\{ \x_C^{(r+1)}, \x_U^{(r+1)} \right\}$ if there is no decomposition. 
The following proofs focus on the no-decomposition case for clarity; however, the extension to the 
decomposable case is trivial since each sub-function of the decomposition is independent. 
We denote applying the subspace optimizer to $f(\x)$ until the stopping criterion is reached as $S_*(f,\x)$
and a single call to the subspace optimizer as $S_1(f, \x)$ and note that $S_*(f,\x)$, 
by definition, returns the global minimum $\x^*$ and that repeatedly calling $S_1(f,\x)$ is equivalent
to calling $S_*(f,\x)$.

%

For convenience, we restate conditions (3.1) and (3.3) from~\citeNP{bonettini2011inexact} (without constraints) on 
the sequence $\{\x^{(k)}\}$ generated by an iterative algorithm on blocks $\x_i$ for $i = 1,\dots,m$, respectively as \\
%
\hsize=\columnwidth
\begin{dmath*}
f( \x_1^{(k+1)}, \dots, \x_i^{(k+1)}, \dots, \x_m^{(k)} ) \leq
    f( \x_1^{(k+1)}, \dots, \x_i^{(k)} + \lambda_i^{(k)}\textbf{\emph{d}}_i^{(k)}, \dots, \x_m^{(k)} ),~~~~~~~~~~~\condOne
\end{dmath*}    
where $\lambda_i^{(k)}$ is computed using Armijo line search and $\textbf{\emph{d}}_i^{(k)}$ is a feasible 
descent direction, and
\begin{dgroup*}
\begin{dmath*}
||\nabla_i f( \x_1^{(k+1)}, \dots, \x_i^{(k+1)}, \dots, \x_m^{(k)} )|| \\ \leq
    \eta|| \nabla_i f( \x_1^{(k+1)}, \dots, \x_{i-1}^{(k+1)}, \dots, \x_m^{(k)} ) || \condition*{i = 1,\dots,m}
\end{dmath*}
\begin{dmath*}
||\nabla_i f( \x_1^{(k+1)}, \dots, \x_i^{(k+1)}, \dots, \x_m^{(k+1)} )|| \\ \leq
    \eta || \nabla_{i-1} f( \x_1^{(k+1)}, \dots, \x_{i-1}^{(k+1)}, \dots, \x_m^{(k)} ) || \condition*{i = 2,\dots,m}
\end{dmath*}
\begin{dmath*}
||\nabla_1 f( \x_1^{(k+2)}, \dots, \x_i^{(k+1)}, \dots, \x_m^{(k+1)} )|| \\ \leq
    \eta^{1-m} || \nabla_{m} f( \x_1^{(k+1)}, \dots, \x_m^{(k+1)} ) ||,  ~~~~~~~~~~~\condTwo
\end{dmath*}
\end{dgroup*}
where $\eta \in [0,1)$ is a forcing parameter. See~\citeNP{bonettini2011inexact} for further details. The inexact 
Gauss-Seidel method is defined as every method that generates a sequence such that these conditions
hold and is guaranteed to converge to a critical point of $f(\x)$ when $m=2$. 
Let \RDIS{DR} refer to RDIS$(f,\x,\x^0,S=\text{DR},\epsilon = 0)$.

\begin{prop}
If the non-restart steps of RDIS satisfy 
\condOne~and \condTwo, $\epsilon = 0$, the number of variables is $n$,
the volume of the global basin is $v = l^n$, and the volume of the entire space is $V = L^n$,
then \RDIS{DR} returns the global minimum after $t$ restarts, 
with probability $1 - (1-(v/V))^t$.
\end{prop}

\begin{proof}
\textbf{Step 1.}
Given a finite number of restarts, one of which starts in the global basin, then 
\RDIS{DR}, with no recursion, returns the global minimum and satisfies 
\condOne~and \condTwo. This can be seen as follows.

At $r=0$, \RDIS{DR}
chooses $\x_C^{(0)} = \x^0$ and $\x_U^{(0)} = \emptyset$
and repeatedly calls $S_1(f^{(0)},\x_C^{(0)})$. This is equivalent to calling $S_*(f^{(0)},\x_C^{(0)}) = S_*(f,\x)$, 
which returns the global minimum ${\x^*}$. Thus, \RDIS{DR} returns 
the global minimum. Returning the global minimum corresponds to a step in the exact 
Gauss-Seidel algorithm,
which is a special case of the IGS algorithm and, by definition, satisfies \condOne~and \condTwo.

\textbf{Step 2.}
Now, if the non-restart steps of $S_1(f,\x)$ satisfy  
\condOne~and \condTwo, then \RDIS{DR} returns the global minimum. We show this
by induction on the levels of recursion.

\emph{Base case.} From Step 1, 
we have that \RDIS{DR}$(f^{(d)}, \x^{(d)})$ returns the global minimum and satisfies
\condOne~and \condTwo, since \RDIS{DR} does not recurse beyond this level. \\
\emph{Induction step.} Assume that \RDIS{DR}$(f^{(r+1)},\x^{(r+1)})$ returns the global minimum. We now show that
\RDIS{DR}$(f^{(r)},\x^{(r)})$ returns the global minimum.
\RDIS{DR}$(f^{(r)}, \x^{(r)})$ first partitions $\x^{(r)}$ into the two blocks $\x_C^{(r)}$ and 
$\x_U^{(r)}$ and then iteratively takes 
the following two steps: $\rho_C^{(r)} \leftarrow S_1(f|_{\sigma^*_U}^{(r)}(\x_C^{(r)}))$ and 
$\rho_U^{(r)} \leftarrow$ \RDIS{DR}$(  f|_{\rho_C}^{(r)}(\x_U)  )$.
The first simply calls the subspace optimizer on $\rho_C^{(r)}$. The 
second is a recursive call equivalent to \RDIS{DR}$(f^{(r+1)},\x^{(r+1)})$, which, 
from our inductive assumption, returns
the global minimum $\rho_U^{(r)} = \x_U^{(r)^*}$ of $f|_{\rho_C}^{(r)}(\x_U)$ and satisfies \condOne~and \condTwo.
For $S_1(f|_{\sigma^*_U}^{(r)}(\x_C^{(r)}))$, \RDIS{DR} will never restart the subspace optimizer 
unless the sequence it is generating
converges. Thus, for each restart, since there are only two blocks and both the 
non-restart steps of $S_1( f|_{\sigma^*_U}^{(r)}(\x_C^{(r)}) )$
and the \RDIS{DR}$(  f|_{\rho_C}^{(r)}(\x_U) )$ steps satisfy \condOne~and \condTwo~then 
\RDIS{DR} is a 2B-IGS
method and the generated sequence converges to the stationary point
of the current basin. At each level, after converging, \RDIS{DR} will restart, iterate until convergence, 
and repeat for a finite number of restarts, one of which will start in the global basin and thus converge to 
the global minimum, which is then returned. 

\textbf{Step 3.}
Finally, 
since the probability of \RDIS{DR} starting in the global basin is $(v/V)$, then the probability of it not starting
in the global basin after $t$ restarts is $(1-(v/V))^t$. 
From above, we have that \RDIS{DR} will return the global minimum if it starts in the global basin, 
thus \RDIS{DR} will return the global minimum after $t$ restarts with probability $1-(1-(v/V))^t$.
\end{proof}

\section{RDIS Subroutine Details}

\subsection{Variable Selection}
In hypergraph partitioning, the 
goal is to split the graph into $k$ components of approximately equal size while minimizing the 
number of hyperedges cut. Similarly, in order to maximize decomposition, RDIS should
choose the smallest block of variables that, when assigned, decomposes the remaining 
variables. Accordingly, RDIS constructs a hypergraph $H = (V,E)$ with a vertex for each 
term, $\{n_i \in V : f_i \in f \}$ and a hyperedge for each variable, $\{e_j \in E : x_j \in \x\}$, 
where each hyperedge $e_j$ connects to all vertices $n_i$ for which the corresponding 
term $f_i$ contains the variable $x_j$. Partitioning $H$, the resulting cutset will be the smallest 
set of variables that need to be removed in order to decompose the hypergraph. And since 
assigning a variable to a constant effectively removes it from the optimization (and the hypergraph), 
the cutset is exactly the set that RDIS chooses on line 2. 

\subsection{Execution time}
Variable selection typically occupies only a tiny fraction of the runtime of RDIS, with the vast majority of
RDIS' execution time spent computing gradients for the subspace optimizer. A small, but non-negligible 
amount of time is spent maintaining the component graph, but this is much more efficient than if we were to
recompute the connected components each time, and the exponential gains from decomposition are well worth 
the small upkeep costs.

\section{Experimental Details}
\label{app:experiments}

All experiments were run on the same compute cluster. Each computer in the cluster 
was identical, with two 2.33GHz quad core Intel Xeon E5345 processors and 16GB of RAM.
Each algorithm was limited to a single thread.


\subsection{Structure from Motion}

In the structure from motion task (bundle adjustment~\cite{triggs00bundle}), 
the goal is to minimize the 
error between a dataset of points in a 2-D image and a projection of fitted 3-D points representing a scene's 
geometry onto fitted camera models. The variables are the parameters of the cameras and the positions of 
the points and the cameras. This problem is highly-structured in a global sense: cameras only interact 
explicitly with points, creating a bipartite graph structure that RDIS is able to exploit. The dataset 
used is the 49-camera, 7776-point data file from the Ladybug 
dataset~\cite{agarwal2010bundle}, where the number of points is scaled
proportionally to the number of cameras used (i.e., if half the 
cameras were used, half of the points were included). There are $9$ variables per camera and
$3$ variables per point.

\subsection{Highly Multimodal Test Function}
\label{app:testfunc}

The test function is defined as follows. Given a height $h$, 
a branching factor $k$, and a maximum arity $a$, we define a 
complete $k$-ary tree of variables of the specified height, with $x_0$ as the root. 
For all paths $p_j \in P$ in the tree of length $l_j \leq a$, with $l_j$ even, 
we define a term $t_{p_j} =\prod_{x_i \in p_j} \sin(x_i)$. 
The test function is
$f_{h,k,a}(x_0, \dots, x_n) = \sum_{i=1}^n c_0 x_i + c_1 x_i^2 + c_2 \sum_{P} t_{p_j}$. 
The resulting function is a multidimensional sinusoid placed in the 
basin of a quadratic function parameterized by $c_1$, with a linear slope 
defined by $c_0$. The constant $c_2$ controls the amplitude of the sinusoids. 
For our tests, we used $c_0 = 0.6, c_1 = 0.1,$ and $c_2 = 12$.
A 2-D example of this function is shown in Figure~\ref{fig:testfuncplot}. 
We used a tree height of $h = 11$, with branching factor $k = 2$, 
resulting in a function of $4095$ variables. We evaluated each of the 
algorithms on functions with terms of arity $a \in \{ 4, 8, 12 \}$, where a larger arity 
defines more complex dependencies between variables as well as more 
terms in the function. 
The functions for the three different arity levels had 16372, 24404, and 30036 terms, respectively.

Figure~\ref{fig:trajectories} shows the value of the current state for
each algorithm over its entire execution. These are the trajectories for Figure~3 of the main paper.

\begin{figure}[t]
\includegraphics[width=0.8\columnwidth]{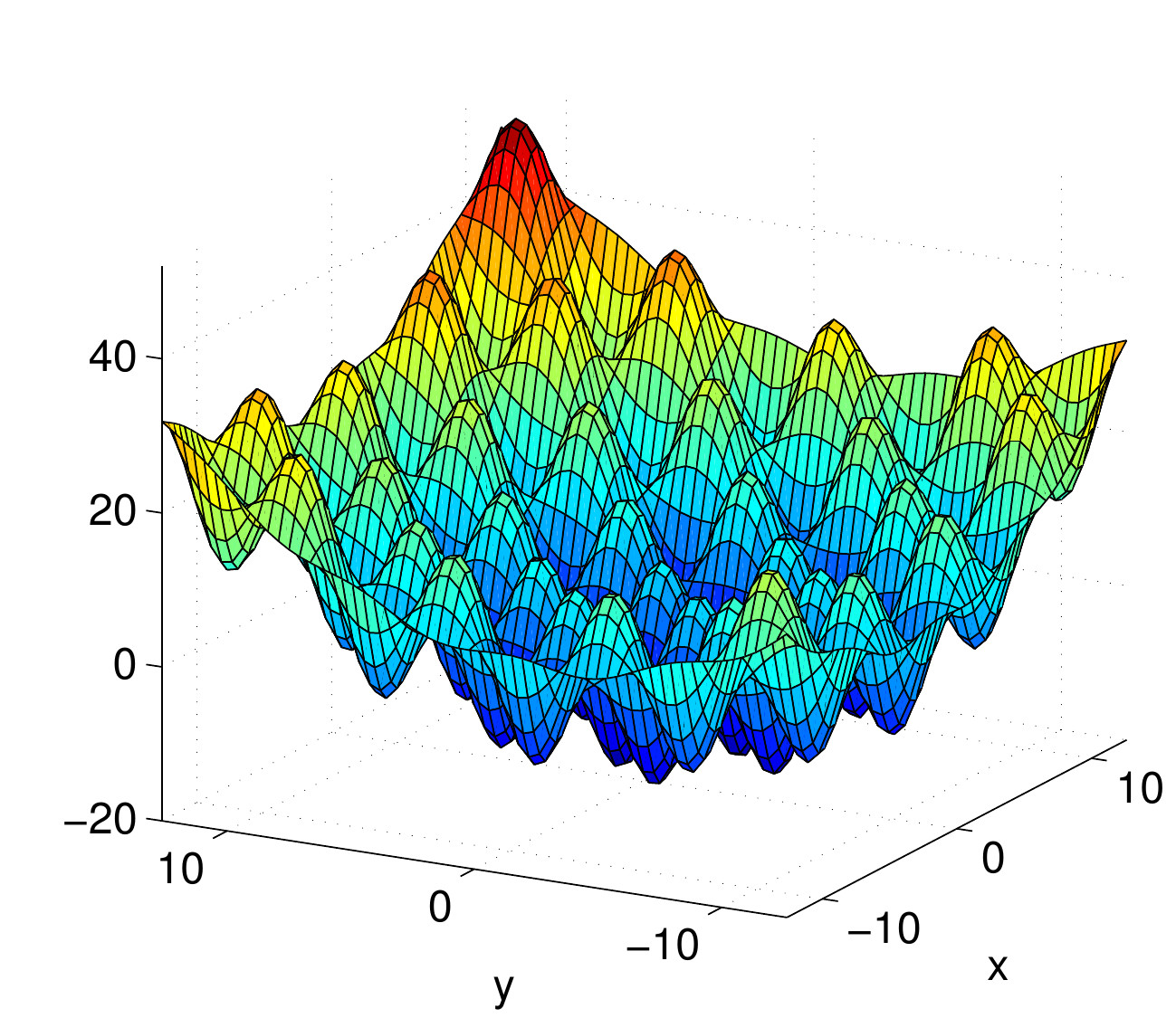}
\caption{\small{A 2-D example of the highly multimodal test function.}
}
\label{fig:testfuncplot}
\end{figure} 

\begin{figure}[h]
\includegraphics[width=\columnwidth]{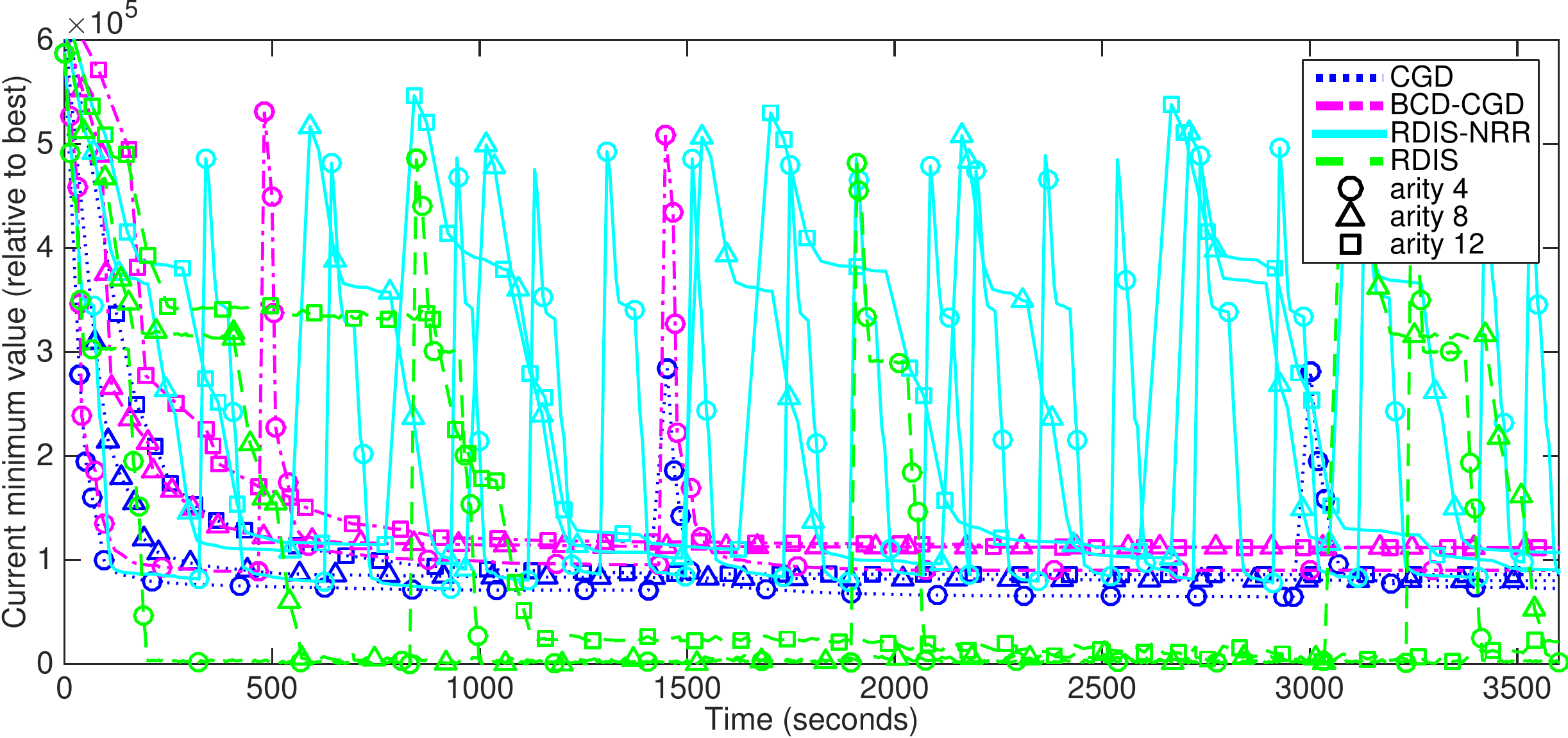}
\caption{\small{
Trajectories on the test function for the data in Figure~3 in the main paper. Sharp rises show restarts. Notably, 
RDIS-NRR restarts much more often than the other algorithms because decomposition allows it to move through
the space much more efficiently. Without internal restarting it gets stuck at the same local minima
as BCD-CGD and CGD. For arity 12, RDIS never performs a full restart and still
finds the best minimum, despite using the same initial point as the other algorithms.
}}
\label{fig:trajectories}
\vskip -0.2in
\end{figure} 

\subsection{Protein Folding}

\newcommand{\goodchi}{\protect\raisebox{2pt}{$\chi$}}

%
%

\subsubsection{Problem details}

Protein folding~\cite{anfinsen73,baker00} is the process by which a protein, consisting of a long chain of amino acids, assumes its functional shape. The computational problem is to predict this final conformation given a known sequence of amino acids. This requires minimizing an energy function consisting mainly of a sum of pairwise distance-based terms representing chemical bonds, hydrophobic interactions, electrostatic forces, etc., where, in the simplest case, the variables are the relative angles between the atoms. The optimal state is typically quite compact, with the amino acids and their atoms bonded tightly to one another and the volume of the protein minimized. Each amino acid is composed of a backbone segment and a sidechain, where the backbone segment of each amino acid connects to its neighbors in the chain, and the sidechains branch off the backbone segment and form bonds with distant neighbors. The sidechain placement task is to predict the conformation of the sidechains when the backbone atoms are fixed in place. 

Energies between amino acids are defined by the Lennard-Jones potential function, as specified in the Rosetta protein folding library~\cite{leaver2011rosetta3}. The basic form of this function is $E_{LJ}(r) = \frac{A}{r^{12}} - \frac{B}{r^6}$, where $r$ is the distance between two atoms and $A$ and $B$ are constants that vary for different types of atoms. The Lennard-Jones potential in Rosetta is modified slightly so that it behaves better when $r$ is very large or very small. The full energy function is $E(\phi) = \sum_{\phi} E_{jk}( R_j(\goodchi_j), R_k(\goodchi_k) )$, where $R_j$ is an amino acid (also called a residue) in the protein, $\phi$ is the set of all torsion angles, and $\phi_i \in \goodchi_j$ are the angles for $R_j$. Each residue has between zero and four torsion angles that define the conformation of its sidechain, depending on the type of amino acid. The terms $E_{jk}$ compute the energy between pairs of residues as $E_{jk} = \sum_{a_j}\sum_{a_k} E_{LJ}(r(a_j(\goodchi_j),a_k(\goodchi_k)))$, where $a_j$ and $a_k$ refer to the positions of the atoms in residues $j$ and $k$, respectively, and $r(a_j,a_k)$ is the distance between the two atoms. The torsion angles define the positions of the atoms through a series of kinematic relations, which we do not detail here.

The smallest (with respect to the number of terms) protein (ID 1) has $131$ residues, $2282$ terms, and $257$ variables, while the largest (ID 21) has $440$ residues, $9380$ terms, and $943$ variables. The average number of residues, terms, and variables is $334$, $7110$, and $682$, respectively. The proteins with their IDs from the paper are as follows:
(1) 4JPB,
(2)    4IYR,
(3)    4M66,
(4)    3WI4,
(5)    4LN9,
(6)    4INO,
(7)    4J6U,
(8)    4OAF,
(9)    3EEQ,
(10)    4MYL,
(11)    4IMH,
(12)    4K7K,
(13)    3ZPJ,
(14)    4LLI,
(15)    4N08,
(16)    2RSV,
(17)    4J7A,
(18)    4C2E,
(19)    4M64,
(20)    4N4A,
(21)    4KMA.

\end{appendices}

\end{document}